\documentclass{article}

\usepackage{ArXiv_style}

\usepackage[utf8]{inputenc} 
\usepackage[T1]{fontenc}    
\usepackage{url}            
\usepackage{booktabs}       
\usepackage{amsfonts}       
\usepackage{amsmath}
\usepackage{hyperref}       
\usepackage{cleveref}
\usepackage{algorithm}
\usepackage{algorithmic}
\usepackage{nicefrac}       
\usepackage{microtype}      
\usepackage{xcolor}         
\usepackage{tikz}
\usepackage{caption}
\usepackage{subcaption}
\usepackage{amsthm}

\theoremstyle{plain}
\newtheorem{theorem}{Theorem}[section]

\theoremstyle{definition}

\theoremstyle{remark}

\newtheorem{example}[theorem]{Example}

\title{EM-NeSy: Expectation Maximization for Neurosymbolic Learning\thanks{Code will be released upon publication.}}

\author{%
  Annegret Seibt \\
  Department of Computer Science\\
  KU Leuven\\
  \texttt{annegret.seibt@kuleuven.be} \\
  \And
  Luc De Raedt \\
  Department of Computer Science\\
  KU Leuven \\
  \texttt{luc.deraedt@kuleuven.be} \\
   \AND
  Giuseppe Marra \\
  Department of Computer Science\\
  KU Leuven\\
  \texttt{giuseppe.marra@kuleuven.be}  \\
}

\begin{document}

\maketitle

\begin{abstract}
        Neurosymbolic (NeSy) models integrate neural networks and symbolic reasoning for robust and interpretable AI.
        State-of-the-art NeSy models require that the symbolic component is expressed in a differentiable way, often complicating the use of approximate inference.
        We propose EM-NeSy which casts probabilistic NeSy learning as an instance of the Expectation-Maximization (EM) algorithm.
        In the expectation step, we compute the posterior over the neurally predicted symbols conditioned on the label via probabilistic inference. In the maximization step, we update the neural parameters based on this posterior using gradient descent only through the neural component.
        This formulation unlocks the full potential of the EM algorithm for NeSy learning.
        It allows NeSy to extend naturally to approximate reasoning without any additional modifications or differentiability requirements of the symbolic component.
        Furthermore, it recovers the standard end-to-end gradient-based NeSy setting under exact inference. 
        Our experimental results demonstrate the scalability and computational efficiency of EM-NeSy.
\end{abstract}

\section{Introduction}
    \label{sec_Intro}
    Neurosymbolic (NeSy) models combine the perceptual strengths of neural networks with the reasoning capabilities of symbolic logic, enabling systems to process raw sensory data while maintaining structured, interpretable reasoning over high-level concepts.
    The integration of continuous neural networks and discrete symbolic logic can be realized through numeric \cite{demeester2016lifted,diligenti2017sbr, yang2017differentiable,minervini2017adversarial, evans2018dilp, dong2019nlm, raghothaman2019difflog, rocktaschelntp, weber2019nlprolog} or fuzzy \cite{sourek2018lrnn, wang2020fuzzy, badreddine2022ltn, pryor2022neupsl} relaxations of logic statements.
    Alternatively, De Raedt \cite{deraedt2019nesydef} coined the formula
    $$\text{Neurosymbolic = Neural + Logical + Probabilistic},$$
    highlighting probabilistic semantics as a natural choice for linking neural and logical components, as they provide a principled way to model uncertainty across different possible neural network outputs.
    This paradigm captures an approach that is now widely adopted by state-of-the-art NeSy models \cite{xu2018semantic,manhaeve2018deepproblog, yang2020neurasp, huang2021scallop, marra2021nmlns, winters2022deepstochlog},
    which typically implement it by performing probabilistic inference over the outputs of neural networks.
    This enables them to train their neural component in an end-to-end fashion -- meaning that they backpropagate gradients, first through the symbolic and then the neural component.
    While probabilities are easily differentiable, exact probabilistic inference is \#P-hard and therefore not always possible --
    and often not necessary during training since sufficient learning signals can be obtained by approximate inference.
    However, when resorting to approximate inference methods, differentiability can become challenging.
    The case of sampling, which is inherently non-differentiable, has been extensively studied and several specialized approaches have been proposed, including gradient estimators \cite{sutton1999reinforce, desmet2023indecater, maene2024weightme} and reparameterization-based techniques \cite{kingma2013vae}.
    But differentiability issues might also arise in iterative approximate inference methods like Pearl's belief propagation algorithm \cite{pearl1994belief} or dual decomposition and linear-programming-based relaxations for approximate maximum-a-posteriori (MAP) inference \cite{jojic2010dual,sontag2011dual} -- which are theoretically differentiable but, in practice, backpropagating through all their iterations would introduce a lot of overhead.

    We propose a novel learning approach for NeSy systems based on the Expectation-Maximization (EM) algorithm which generally decomposes into alternating expectation and maximization steps \cite{dempster1977em}.
    In every expectation step (E-step), a posterior distribution over the neurally predicted symbolic variables conditioned on the target variable label is computed.
    In every maximization step (M-step), the parameters of the neural component are then updated based on the posterior computed in the E-step via gradient descent solely through the neural component itself.
    Hence, irrespective of differentiability, the E-step can be carried out using any symbolic inference engine -- whether exact or approximate. 
    This is a major operational advantage, as any specialized and optimized software can be employed, regardless of the symbolic framework.
    Furthermore, EM-NeSy recovers the standard NeSy setting under exact reasoning.
    
    In summary, our work makes the following key contributions:
\begin{itemize}
    \item We recast NeSy learning within the -- traditionally fully probabilistic -- Expectation-Maximization (EM) algorithm framework, enabling the use of advanced EM variants such as Generalized EM \cite{dempster1977em}, Stochastic EM \cite{celeux1996stochastic} or Hard EM \cite{ruggieri2020hard} for improved flexibility and scalability;
    \item We propose a general framework for integrating any probabilistic inference engine into a NeSy model without requiring differentiability of the symbolic inference process;
    \item We provide experimental evaluation demonstrating the effectiveness and generality of the proposed approach across diverse exact and approximate inference methods.
\end{itemize}
    
    This paper is structured as follows:
    In \cref{sec_prelim}, we review the standard NeSy learning setting and the EM algorithm.
    \Cref{sec_emnesy} presents our EM-NeSy learning framework, including a proof of its equivalence to standard NeSy learning under exact inference and an EM-specific approximation to accelerate convergence.
    \cref{sec_general} demonstrates how the framework generalizes to alternative inference schemes, with a particular focus on sampling-based and iterative approaches.
    In \cref{sec_related} we discuss related work.
    Finally, \cref{sec_experiments} reports experimental results.

\section{Preliminaries}
    \label{sec_prelim}

\subsection{The Expectation-Maximization algorithm}
    The Expectation-Maximization (EM) algorithm, introduced by Dempster et al.~\cite{dempster1977em}, is an iterative algorithm for computing maximum likelihood estimates in parameterized probabilistic models with latent variables.
    Consider a latent variable model parameterized by $ \boldsymbol \theta$ with observed variables $\mathbf Y$ and latent variables $\mathbf Z$. Let $\mathbf y$ denote an instantiation of $\mathbf Y$ and $R(\mathbf Z)$ the range -- i.e.\ the set of all instantiations -- of $\mathbf Z$.
    Direct maximization of the marginal likelihoods $
        \log p(\mathbf y \mid \boldsymbol \theta) = \log \sum_{\mathbf z \in R(\mathbf Z)} p(\mathbf y, \mathbf z \mid \boldsymbol \theta) $
    is often intractable due to the summation over the latent variables $\mathbf Z$.
    EM addresses this by alternating between two steps:
    The E-step computes the expected complete-data log-likelihood under the current parameter estimate $\boldsymbol \theta^{(t)}$
    \begin{equation}
        \label{Estep}
        Q(\boldsymbol \theta \mid\boldsymbol \theta^{(t)}) = \mathbb{E}_{\mathbf z \sim \mathbf Z \mid \mathbf Y, \boldsymbol \theta^{(t)}} \big[ \log p(\mathbf Y, \mathbf z \mid \boldsymbol \theta) \big].
    \end{equation}
    Note that \( Q(\boldsymbol \theta \mid \boldsymbol \theta^{(t)}) \) is a function of \(\boldsymbol \theta\); therefore, the actual computation in the E-step consists of determining the posterior distribution $ p(\mathbf Z \mid \mathbf Y, \boldsymbol \theta^{(t)}) $.
    The M-step updates the parameters by maximizing the expectation \eqref{Estep}:
    \begin{equation}
        \label{Mstep}
    \boldsymbol \theta^{(t+1)} = \arg \max_{\boldsymbol \theta} Q(\boldsymbol \theta \mid \boldsymbol \theta^{(t)}).
    \end{equation}
    These steps are repeated until convergence to a local optimum. The algorithm guarantees that the likelihood does not decrease at each iteration and is widely used in probabilistic models, such as mixture models and hidden Markov models.

    A common extension is the Generalized EM (GEM) algorithm \cite{dempster1977em}, which relaxes the M-step requirement: instead of fully maximizing the $Q$-function \eqref{Estep}, GEM only requires an update that increases it.
    This flexibility enables variants such as
    \begin{itemize}
        \item \textbf{Hard EM} \cite{ruggieri2020hard} which instead of computing the expectation in \eqref{Estep}, only computes the MAP assignment $
            Q(\boldsymbol \theta \mid \boldsymbol \theta^{(t)})_{\text{hard}}
            \;=\;
            \max_{\mathbf z \in R(\mathbf Z)} \; \log p(\mathbf Y, \mathbf z \mid \boldsymbol \theta);$
        \item \textbf{Online EM} \cite{cappe2009online} or \textbf{Stochastic EM} \cite{celeux1996stochastic}
        which only use a single observation $\{\mathbf Y=\mathbf y\}$ or mini-batches instead of the entire dataset to compute the $Q$-function $            Q(\boldsymbol \theta \mid \boldsymbol \theta^{(t)})_{\text{online}} = \mathbb{E}_{\mathbf Z \mid \mathbf Y=\mathbf y, \boldsymbol \theta^{(t)}} \big[ \log p(\mathbf Y=\mathbf y, \mathbf Z \mid \boldsymbol \theta) \big]$
        in the E-step.
        Consequently, it only performs a partial update $
            \boldsymbol \theta^{(t+1)}_{\text{online}} = \boldsymbol \theta^{(t)} + \gamma^{(t)} \bigl( {\boldsymbol \theta}^{(t+1)} - \boldsymbol \theta^{(t)} \bigr)$
        in the M-step for some learning rate $\gamma^{(t)}$, making the algorithm suitable for large-scale or streaming data -- and also to NeSy learning;
        \item \textbf{Multi-M EM}, a variant of online EM that performs multiple partial M-steps per E-step.
    \end{itemize}

\subsection{Probabilistic NeSy Learning as a Latent Variable Model}
    \label{sec_latent}
    State-of-the-art NeSy models combine a neural and a symbolic component by performing probabilistic inference over the outputs of neural networks to predict the distribution of a queried target variable \cite{manhaeve2018deepproblog, huang2021scallop, winters2022deepstochlog}.
    
    \begin{example}[Visual Sudoku \cite{augustine2022visual}]
    Suppose, for instance, we have a dataset where the input is a $9 \times 9$ grid of handwritten digits and the label is whether they form a valid Sudoku but there is no direct supervision on single digits.
    A natural approach to train a neurosymbolic model on such a dataset is as follows (see \cref{fig:visudo}):
    First we employ a neural component to classify each handwritten digit, producing a probability distribution over the possible values $\{0,\dots,8\}$ for each digit.
    Then a symbolic component that encodes the rules of Sudoku can use these probabilistic digit predictions to compute the likelihood that the entire grid constitutes a valid Sudoku configuration.
    Finally, this likelihood can be compared to the provided label (valid or invalid) to define a loss function, allowing the entire system to be trained end-to-end.
    In other words, the model learns to adjust the neural predictions so that the probability of satisfying the Sudoku constraints aligns with the ground truth.
    \begin{figure}[t]
      \centering
      \begin{subfigure}[t]{0.16\textwidth}
        \centering
        \includegraphics[width=\linewidth]{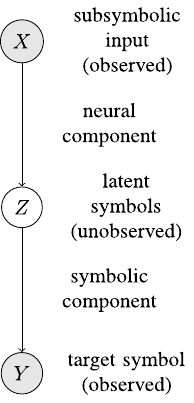}
        \caption{The probabilistic NeSy model is a latent variable model.}
        \label{fig_nesy_model}
      \end{subfigure}\hfill
      \begin{subfigure}[t]{0.80\textwidth}
        \centering
        \includegraphics[width=\linewidth]{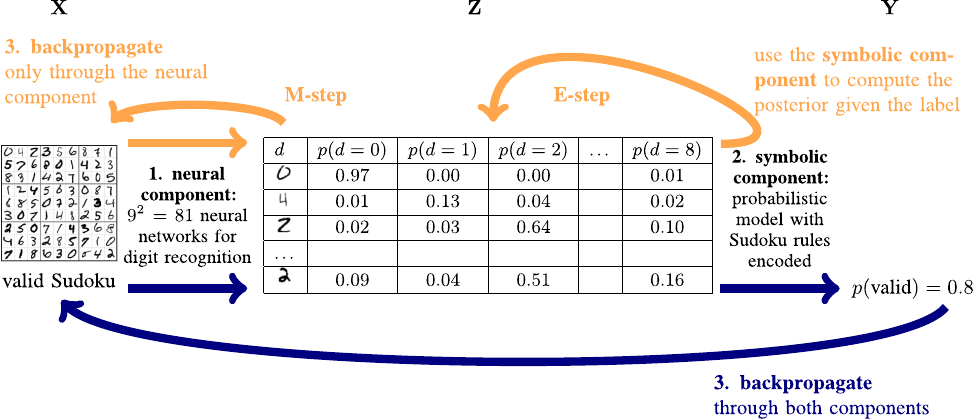}
        \caption{
        To train on Visual Sudoku~\cite{augustine2022visual}, an 
        \textcolor{blue!70!black}{\textbf{end-to-end differentiable NeSy}} model neurally predicts 81 digit distributions, computes the Sudoku-validity probability via the symbolic model, and backpropagates through both parts.
        \textcolor{orange!70!white}{\textbf{EM-NeSy}} also neurally predicts prior distributions for all digit, but uses the symbolic model to obtain posterior digit distributions conditioned on the label and trains the neural predictors by matching priors to posteriors, avoiding gradients through the symbolic component. (Sudoku grid and digits from~\cite{augustine2022visual}.)}
        \label{fig:visudo}
      \end{subfigure}
      \caption{Comparison of end-to-end differentiable NeSy learning and EM-NeSy learning.}
      \label{fig:main}
    \end{figure}
    \end{example}

    Consider the general learning task of predicting a symbolic output variable $\mathbf Y$ from a subsymbolic input variable $\mathbf X$, given a dataset of observed samples and some background knowledge.
    Then we can interpret a NeSy model solving this task as a latent variable model (see \cref{fig_nesy_model}) with
    a subsymbolic input variable $\mathbf X$, a hidden symbolic variable $\mathbf Z$ \footnote{In the following, we assume $\mathbf{Z}$ to be discrete. The theory developed, however, extends naturally to the continuous case by replacing summations with integrals.} representing the output of the neural component and a symbolic output variable $\mathbf Y$ that can be inferred from $\mathbf Z$ using the symbolic component that encodes the background knowledge.

    The symbolic component is a probabilistic model that predicts the target variable $\mathbf Y$ from the latent symbolic variable $\mathbf Z$ via probabilistic inference, i.e. by computing the distribution $p_{\text{Sy}}(\cdot |\mathbf z)$ of the output variable $\mathbf Y$ conditioned on the latent variable $\mathbf Z$ taking the state $\mathbf z$.
    Meanwhile, the neural component serves as a conditional prior for the symbolic component, by predicting the distribution $p_{\text{Ne}}(\cdot |\mathbf x, \boldsymbol \theta)$ of the latent variable $\mathbf Z$ given the subsymbolic input $\mathbf x$ via neural networks $\mathbf N$ parameterized by $\boldsymbol \theta$.

    Thus, we can abstractly represent the neurosymbolic predictor as a conditional probability distribution
    \begin{align}
        p(\mathbf y|\mathbf x,\boldsymbol \theta)
        = \sum_{\mathbf z \in R(\mathbf Z)} p(\mathbf y,\mathbf z|\mathbf x,\boldsymbol \theta)
        = \sum_{\mathbf z \in R(\mathbf Z)} p_{\text{Ne}}(\mathbf z|\mathbf x,\boldsymbol \theta)p_{\text{Sy}}(\mathbf y|\mathbf z), \label{eq_NeSy_pred}
    \end{align}
    where we denote by $R(\mathbf Z)= R(Z_1) \times \dots \times R(Z_n)$ the range (i.e.\ the set of all possible values) of $\mathbf Z$.
    The loss function in the standard NeSy setting is defined on the target-variable level as the negative log-likelihood
    \begin{align}
        \mathcal L(\boldsymbol\theta) = - \log p(\mathbf y|\mathbf x,\boldsymbol \theta)
        = - \log \sum_{\mathbf z \in R(\mathbf Z)} p_{\text{Ne}}(\mathbf z|\mathbf x, \boldsymbol\theta) p_{\text{Sy}}(\mathbf y|\mathbf z). \label{std_loss}
    \end{align}
    It is typically minimized by backpropagating gradients through the symbolic and the neural component.
    Note that learnable parameters could also be introduced to the symbolic component and be jointly optimized with the parameters of the neural component.


\section{EM-NeSy Learning}
    \label{sec_emnesy}
    In this section, we will follow the interpretation of probabilistic NeSy as a conditional latent variable model from \cref{sec_latent} and introduce our EM-NeSy learning framework.
    Compared to the standard NeSy setting, we will change the use of the symbolic component during learning:
    instead of letting it predict the log-likelihood of the label given the neurally predicted latent variables, we will use it to compute the posterior distribution of these neurally predicted latent variables given the label:
    \begin{equation}
        \label{eq_latent_posterior}
        p(\mathbf z|\mathbf y,\mathbf x,\boldsymbol \theta) = \frac{p_{\text{Ne}}(\mathbf z|\mathbf x, \boldsymbol \theta)p_{\text{Sy}}(\mathbf y|\mathbf z)}{p(\mathbf y|\mathbf x,\boldsymbol \theta)}.
    \end{equation}
    Hence, the symbolic component in EM-NeSy learning acts as a probabilistic corrector, refining and aligning the neural network outputs with the background knowledge.
    It can be thought of as abducing labels
    \begin{equation}
        \label{pseudo_label}
        \ell(\mathbf z) = p(\mathbf z|\mathbf y,\mathbf x,\boldsymbol \theta)
    \end{equation}
    for the neurally predicted latent variables $\mathbf Z$ from the prior distribution $p_{\text{Ne}}(\cdot|\mathbf x, \boldsymbol \theta)$ and the original label $\mathbf y$ of the target variable, which is related to abductive reasoning \cite{poole1993probabilistic, de2007problog} in general.
    Note that these abduced labels can be soft probabilistic labels as in \eqref{pseudo_label} or alternatively hard labels -- in the case of MAP inference: $\ell(\mathbf z) = \mathbb I(\mathbf z = {\arg \max}_{\mathbf z'} p(\mathbf z'|\mathbf y,\mathbf x,\boldsymbol \theta))$, where $\mathbb I$ denotes the indicator function.
    Such a MAP-based abduction is for example performed by Relational Neural Machines \cite{marra2020rnm}.
    This symbolic correction step will be the E-step of our EM-NeSy learning algorithm.
    It can be carried out by any exact or approximate symbolic inference engine and regardless of its differentiability.

    The labels abduced in the E-step can now serve for defining a loss function on the latent-variable level:
    the cross entropy between the neurally predicted prior distributions $p_{\text{Ne}}(\mathbf z|\mathbf x,\boldsymbol \theta)$ of the latent variables and their posterior $\ell(\mathbf z)$:
    \begin{equation}
        \label{em_loss}
        \mathcal L_{\text{EM-NeSy}}(\boldsymbol \theta) = - \sum_{\mathbf z \in R(\mathbf Z)} \ell(\mathbf z) \log p_{\text{Ne}}(\mathbf z|\mathbf x,\boldsymbol \theta).
    \end{equation}
    The M-step of our EM-NeSy learning algorithm uses this loss function to updates the parameters $\boldsymbol \theta$ of the neural component by backpropagating only through the neural component itself, while treating the posterior $\ell(\mathbf z)$ as a constant.
    The entire EM-NeSy algorithm then simply consists of alternating the expectation and maximization steps. For a high-level visualization, see \cref{fig_nesy_model}.

    \label{sec_emnesyproof}
    Under exact inference, EM-NeSy learning recovers the standard NeSy learning setting described in \cref{sec_latent} -- in the sense that it backpropagates the same gradients through the neural component.

    \begin{theorem}
        \label{equivproof}
        The gradients with respect to the neural parameter $\theta$ of the loss $\mathcal L(\theta)$ in the standard NeSy learning setting of \cref{std_loss} and the loss $\mathcal L_{\text{EM-NeSy}}(\theta)$ in the EM-NeSy learning setting of \cref{em_loss} are equal, i.e.\ 
        \begin{equation}
            \nabla_{\boldsymbol \theta} \mathcal L(\boldsymbol \theta) = \nabla_{\boldsymbol \theta} \mathcal L_{\text{EM-NeSy}}(\boldsymbol \theta).
        \end{equation}
    \end{theorem}

    \begin{proof}
        see Appendix \ref{app_equivproof}
    \end{proof}

    \label{emnesy_versions}
    In practice, EM-NeSy learning follows an online-EM style approach, where each E-step processes a single training pair $(\mathbf x,\mathbf y)$ (or a mini-batch) rather than the entire dataset, and the M-step performs a partial parameter update, scaled by the learning rate of the neural component.
    As noted earlier, the abduced labels \eqref{pseudo_label} may be either soft (probabilistic) or hard (MAP), leading to a soft or hard variant of EM-NeSy learning.
    
    A further approximation of EM-NeSy learning, besides the possible use of an approximate inference component can be achieved by a Multi-M EM approach (see Appendix \ref{app_multim}): 
    Every additional backpropagation step for the same training example $(\mathbf x,\mathbf y)$ with regard to the same loss function saves one potentially expensive re-computation of the posterior of the latent variables.
    Note that performing these multiple small steps leads to a slightly more stable learning process than just increasing the learning rate as the latter can cause instability due to overshooting the optimum.
    Finally, also in EM-NeSy, learnable parameters can be introduced into the symbolic component and optimized jointly with the neural component’s parameters that predict the latent variables $\mathbf Z$ \cite{marra2020rnm}.

\section{Generalization to alternative inference schemes}
    \label{sec_general}
    To update the neural parameters in the M-step, we only need the posteriors \eqref{eq_latent_posterior} of the latent variables, i.e.\ the outcome of the probabilistic inference in the E-step, but not the gradients of this computation.
    Hence, the E-step can be carried out by any symbolic inference engine that can compute posterior marginals or the MAP assignment, regardless of whether it performs exact or approximate inference in a differentiable or non-differentiable fashion.
    \textit{Thus, in general, EM-NeSy learning is applicable to any inference scheme, whether differentiable or not}.
    In this section, we will explore when it is especially useful.
    
\subsection{Exact inference}
    In NeSy models, exact inference is often performed by compiling the symbolic component into an arithmetic circuit in an offline knowledge‑compilation step \cite{chavira2008AC,xu2018semantic,manhaeve2018deepproblog,ahmed2022spl}.
    When this is feasible, gradient backpropagation through the symbolic component becomes highly efficient, as it only requires a single backward pass through the circuit.
    Recall from \cref{sec_emnesyproof} that we established the equivalence between the gradients of the standard-NeSy loss and the EM-NeSy loss. This implies that differentiating through the symbolic component is effectively the same as computing posterior marginals.
    So, put differently, backpropagation through the circuit provides a fast and direct way to obtain the posterior marginals required for EM-NeSy learning.
    This link between differentiating and computing posterior marginals was already noted by Darwiche \cite{darwiche2003differential}, who leveraged the same equivalence for the opposite purpose: proposing a differentiable approach to efficiently compute marginals in Bayesian networks.
    Consequently, when exact inference with arithmetic circuits is possible, EM-NeSy does not offer additional benefits in terms of differentiability, as there are no differentiability challenges in this setting.
    
    However, knowledge compilation can be prohibitively expensive, and in some cases a circuit may fail to compile within a reasonable time even though other exact inference methods remain feasible. In such situations, EM‑NeSy can be used to eliminate any potential differentiability issues.
    Furthermore, even if knowledge compilation is possible but perhaps yields a large arithmetic circuit, Multi-M EM-NeSy learning may still be useful to accelerate convergence without introducing instability issues, such as those caused by increasing the learning rate.

\subsection{Sampling-based inference}
    Although sampling is inherently non-differentiable, several gradient estimators have been proposed for sampling-based inference in NeSy models.
    The most straightforward approach is the score function estimator, commonly known as REINFORCE \cite{sutton1999reinforce}.
    This method relies on a simple application of the ``delta-log trick'' \eqref{log_chain_rule} to rewrite the gradient of the log-likelihood of the label as
    \begin{equation}
        \nabla_{\boldsymbol \theta} p(\mathbf y|\mathbf x,\boldsymbol \theta)
        \overset{\eqref{eq_NeSy_pred}}= \sum_{\mathbf z \in R(\mathbf Z)}p_{\text{Sy}}(\mathbf y|\mathbf z) \nabla_{\boldsymbol \theta} p_{\text{Ne}}(\mathbf z|\mathbf x,\boldsymbol \theta)
        \overset{\eqref{log_chain_rule}}=\mathbb E_{\mathbf z \sim p_{\text{Ne}}(\mathbf z|\mathbf x,\boldsymbol\theta)}[p_{\text{Sy}}(\mathbf y|\mathbf z) \nabla_{\boldsymbol \theta} \log p_{\text{Ne}}(\mathbf z|\mathbf x,\boldsymbol \theta)], \label{reinforce}
    \end{equation}
    thus enabling sampling from the neurally predicted distribution $p_{\text{Ne}}(\mathbf z|\mathbf x,\boldsymbol \theta)$ of the latent variables.
    However, because each sample is weighted by $p_{\text{Sy}}(\mathbf y|\mathbf z) \nabla_{\boldsymbol \theta} \log p_{\text{Ne}}(\mathbf z|\mathbf x,\boldsymbol \theta)$, the variance of REINFORCE can become large.
    Lower-variance alternatives include a Rao–Blackwellized version of REINFORCE \cite{desmet2023indecater} and model-sampling \cite{maene2024weightme, verreet2024exal}, which reduce variance by exploiting conditional expectations or structured sampling strategies.
    Still, these methods fundamentally rely on reformulating the neurosymbolic gradient $\nabla_{\boldsymbol \theta} p(\mathbf y|\mathbf x,\boldsymbol \theta)$.
    In contrast, NeSy-EM learning avoids such case-specific rewritings. It only requires estimates of the posterior marginals $p(\mathbf z|\mathbf y,\mathbf x,\boldsymbol \theta)$, which can be obtained using any sampling method -- such as rejection sampling, importance sampling, Gibbs sampling or Metropolis-Hastings \cite{koller2009pgm} -- without additional adjustments to the NeSy setting.

    \paragraph{Rejection sampling.}
    The most straightforward way to approximate a posterior distribution by sampling is rejection sampling.
    It operates by drawing candidate samples from the prior distribution and then accepting or rejecting each sample based on whether it agrees with the observed data.
    Only the accepted samples are kept, and these form an empirical approximation of the posterior.
    If we employ rejection sampling to EM-NeSy, then we almost recover the REINFORCE gradient estimator:
    \begin{align*}
        \label{em_loss}
        \nabla_\theta \mathcal L_{\text{EM-NeSy}}(\boldsymbol \theta)
        &\overset{\eqref{em_loss}, \eqref{eq_latent_posterior}}= - \sum_{\mathbf z \in R(\mathbf Z)} \frac{p_{\text{Ne}}(\mathbf z|\mathbf x, \boldsymbol \theta)p_{\text{Sy}}(\mathbf y|\mathbf z)}{p(\mathbf y|\mathbf x,\boldsymbol \theta)} \nabla_\theta \log p_{\text{Ne}}(\mathbf z|\mathbf x,\boldsymbol \theta)\\
        &=\frac 1 {p(\mathbf y|\mathbf x,\boldsymbol \theta)}\cdot \mathbb E_{\mathbf z \sim p_{\text{Ne}}(\mathbf z|\mathbf x,\theta)}[p_{\text{Sy}}(\mathbf y|\mathbf z) \nabla_{\boldsymbol \theta} \log p_{\text{Ne}}(\mathbf z|\mathbf x,\boldsymbol \theta)].
    \end{align*}
    We would sample from the same distribution as in \eqref{reinforce} and divide the estimates for all gradients by the likelihood of the observed label.
    This adjustment reflects the fact that REINFORCE approximates the gradient of the likelihood itself, whereas EM‑NeSy targets the gradient of the log‑likelihood which differs by the factor ${p(\mathbf y|\mathbf x,\boldsymbol \theta)}$ (see \eqref{log_chain_rule}).

    \paragraph{ABC rejection sampling.}
    \label{abc}
    Rejection sampling -- and thus also REINFORCE -- struggle in large sampling spaces because it rarely produces samples that satisfy the evidence.
    The same is the case for likelihood weighting~\cite{koller2009pgm} under determinism as the likelihood weights then become deterministic as well.
    Approximate Bayesian Computation (ABC) sampling~\cite{tavare1997inferring, beaumont2002approximate}, which is closely related to Pseudo maximum-likelihood estimation~\cite{white1982maximum, gourieroux1984pseudo}, provides a more flexible alternative to likelihood weighting.  
    In likelihood weighting, we sample from the neurally predicted prior $p_{\textbf{Ne}}(\cdot \mid \mathbf x,\boldsymbol \theta)$ and weigh each sample $\mathbf z$ with the likelihood value $p_{\text{Sy}}(\mathbf y \mid \mathbf z)$.  
    However, if the likelihood is intractable to compute or deterministic -- as in the case of Sudoku -- the likelihood weights become deterministic as well and tend to collapse to zero because the evidence is extremely sparse.  
    (Sampling a valid $9 \times 9$ Sudoku configuration is highly unlikely.)

    ABC sampling addresses this issue by introducing
        \emph{summary statistics} $s_{\mathbf{Z}} : R(\mathbf{Z}) \rightarrow \mathcal{S}$ and $s_{\mathbf{Y}} : R(\mathbf{Y}) \rightarrow \mathcal{S}$, which map the potentially complex instantiations of $\mathbf{Z}$ and $\mathbf{Y}$ into a simpler comparison space $\mathcal{S}$ (e.g. $\mathcal S = \mathbb R$),
        and a \emph{distance metric} $\rho : \mathcal{S} \times \mathcal{S} \rightarrow \mathbb{R}$, such as the $L^1$- or $L^2$-norm. 
    In ABC sampling, the exact likelihood weight is replaced by the more flexible discrepancy measure $\frac 1 C \rho\!\left(s_{\mathbf{Z}}(\mathbf z),\, s_{\mathbf{Y}}(\mathbf y)\right)$ (where $C$ denotes a normalization constant), which allows inference to proceed even when the likelihood is unavailable or degenerate.
    A general strategy for selecting summary statistics and discrepancy measures is as follows:
    We decompose the constraint determining whether a sample is valid into a set of constraints $M$.
    For a given sample \(\mathbf z\), let \(k\) denote the number of constraints in \(M\) that it satisfies, and let $s_{\mathbf Z}(\mathbf z):= \frac k {|M|}$.
    By choosing $s_{\mathbf Y}(\mathbf y)=0$ and $\rho = L^1$, the weight assigned to the sample $\mathbf z$ is $\frac 1 C \rho(s_{\mathbf Z}(\mathbf z),s_{\mathbf Y}(\mathbf y)) = \frac 1 C |\frac k {|M|} - 0| = \frac k {C|M|}$.
    Thus, the weight of a sample corresponds to the fraction of constraints it satisfies.
    
    \begin{example}
    \label{example_sudoku}
        In the case of Sudoku, \(M\) may consist of the constraints that certain pairs of digits -- namely those sharing a row, column, or box -- must be different. For an addition task, \(M\) could contain constraints specifying that each individual digit of the sum is correct. In a path-finding task on a grid, \(M\) could consist of constraints indicating whether each grid cell lies on the path or not.
    \end{example}

    A slightly more refined variant of ABC sampling can be obtained by decomposing the sample \(\mathbf z\) into its individual components \((z_1, \dots, z_n)\).
    For each component \(z_i\), we then consider only the subset \(M_i \subset M\) of constraints that depend most on \(Z_i\).
    Let \(k\) denote the number of constraints in \(M_i)\) that are satisfied by \(z_i\).
    We can then assign to each component \(z_i\) the weight
    $
    \frac 1 C \rho\bigl(s_{Z_i}(z_i), s_{\mathbf Y}(\mathbf y)\bigr)
    = \frac 1 C | \frac{k}{|M_i|} - 0 |
    = \frac{k}{C|M_i|},
    $
    which corresponds to the fraction of \emph{relevant} constraints that \(z_i\) satisfies.
    
    For more details, see \cref{app_abc}.

    \subsection{Iterative inference algorithms}
    Differentiability challenges can also arise in iterative approximate inference methods such as belief propagation \cite{pearl1994belief}, or dual decomposition and linear-programming-based relaxations for MAP inference \cite{jojic2010dual,sontag2011dual}. 
    While these algorithms are differentiable in theory, backpropagating through their iterations in practice often requires maintaining a large computation graph or resorting to expensive techniques such as gradient checkpointing or implicit differentiation.
    In contrast, EM-NeSy learning removes this overhead entirely, as it only requires posterior marginals rather than differentiating through the inference procedure.
    Conveniently, belief propagation computes all posterior marginals or MAP assignments during inference, making them immediately available for use without any additional steps.
    
\section{Related Work}
    \label{sec_related}
        While the equivalence of computing gradients and posterior marginals on which the proof in \cref{sec_emnesyproof} is based, has already been exploited by Darwiche to substitute exact inference with differentiation in Bayesian networks \cite{darwiche2003differential}, we use it to do the converse: to replace differentiation by (approximate) inference.

        A notable special case of EM-NeSy learning is found in Relational Neural Machines (RNMs) \cite{marra2020rnm}, a neurosymbolic framework that jointly learns neural and symbolic parameters while performing relaxation-based approximate MAP inference via gradient descent.
        Although RNMs employ the EM algorithm to avoid the computational overhead of differentiating through this iterative inference procedure, they do not introduce EM as a general framework for NeSy learning.

        A closely related framework to EM-NeSy learning is Implicit Maximum Likelihood Estimation (I-MLE) \cite{niepert2021implicit}, a general framework for backpropagating through discrete exponential-family distributions.
        I-MLE operates by correcting the neurally predicted distribution $p_{\boldsymbol \theta}(\cdot)$ by inferring from the label $\mathbf y$ a \emph{target distribution} $q_{\boldsymbol \theta'}(\cdot)$, and then minimizing the KL-divergence between $p_{\boldsymbol \theta}(\cdot)$ and $q_{\boldsymbol \theta'}(\cdot)$ --
        an approach conceptually similar to EM-NeSy learning.
        However, EM-NeSy offers an actual general-purpose choice for the target distribution $q_{\boldsymbol \theta}$: the posterior $p(\mathbf z|\mathbf y,\mathbf x,\boldsymbol \theta)$, thereby establishing a direct connection to the well-studied EM algorithm.
        Moreover, EM-NeSy provides a more general framework: while I-MLE prescribes a specific perturb-and-MAP inference strategy which relies on the availability of a MAP-solver, EM-NeSy simply requires an approximation of the posterior, a problem that has been extensively studied. This flexibility allows EM-NeSy to integrate any black-box inference engine.

        Another related neurosymbolic framework is EXAL \cite{verreet2024exal}, which is based on sampling and follows an approach similar to I-MLE.
        Given the label $\mathbf y$, the method samples a set of explanations, namely instantiations of the latent variables $\mathbf Z$.
        Reweighting these samples using the neurally predicted distribution $p_{\text{Ne}}(\mathbf z | \boldsymbol\theta)$ yields an \emph{auxiliary distribution} $q_{\boldsymbol\theta}(\mathbf x)$.
        The model parameters are then optimized by minimizing the KL divergence between $p$ and $q$.
        EM‑NeSy replaces this auxiliary distribution with the full posterior, making EXAL an instance of the EM‑NeSy framework that approximates the posterior via the auxiliary distribution.
        
        While state-of-the-art NeSy models realize the integration of neural perception and symbolic reasoning by backpropagating gradients through the symbolic component \cite{xu2018semantic,manhaeve2018deepproblog,huang2021scallop,ahmed2022spl},
        Abductive Learning (ABL) \cite{dai2019bridging,jia2025psp,hu2025efficient} sidesteps the need for a differentiable symbolic component in a similar way like EM-NeSy learning:
        The symbolic component corrects neural predictions by abductive logic reasoning.
        While the neural predictions can be probabilities like in EM-NeSy, ABL is not a probabilistic model as it only uses these probabilities or weights to determine which neural outputs will be corrected, but the reasoning itself remains deterministic \cite{jia2025psp}, while in EM-NeSy the abduction itself is probabilistic.
    
\section{Experiments}
    \label{sec_experiments}
    
    To assess the scalability and flexibility of EM-NeSy, we conduct a set of experiments addressing the following questions:
    \textbf{Q1}: Under \emph{exact inference}, does EM-NeSy recover the same learning signal as end-to-end differentiable NeSy while reducing memory and runtime?
    \textbf{Q2}: How does EM-NeSy behave when the E-step is performed with \emph{approximate or non-differentiable inference}?

\subsection{Benchmarks}
    To answer our questions, we conduct experiments on three widely used NeSy benchmarks that all present weakly supervised tasks combining visual perception with symbolic reasoning.
    
    \textbf{Multi-Digit MNISTAdd \cite{manhaeve2018deepproblog}.}
    Each input consists of a sequence of \(2n\) MNIST~\cite{deng2012mnist} images, for some \(n \geq 1\), representing an addition task where two numbers of equal length are summed (e.g., \(495 + 037\) for $n=3$).
    The dataset provides only the final sum as the label for each sequence, without individual digit annotations.\\
    \textbf{Visual Sudoku. \cite{augustine2022visual}}  
    Each input corresponds to a \(4 \times 4\)- or \(9 \times 9\)-grid of MNIST images representing a Sudoku puzzle. Supervision is limited to a single binary label indicating whether the grid constitutes a valid Sudoku configuration, without explicit guidance on individual cell values.\\
    \textbf{Warcraft path-planning \cite{vlastelica2019differentiation}.} Each input consists of a visual grid of $12 \times 12$- or $30 \times 30$ cells of five different terrain types with different costs of traversing them. Supervision is a binary mask indicating the shortest path from the top-left to the bottom-right corner.

\subsection{Baselines}
We compare EM‑NeSy against state‑of‑the‑art NeSy models, namely A‑NeSI~\cite{krieken2023anesi}, I‑MLE~\cite{niepert2021implicit},
EXAL~\cite{verreet2024exal},
Scallop~\cite{huang2021scallop} and exact inference based on knowledge compilation \cite{winters2022deepstochlog,xu2018semantic,ahmed2022spl}, noting that not all methods are evaluated on every benchmark.
The exact inference methods we use (namely DeepStochLog (DSL)~\cite{winters2022deepstochlog}, Semantic Loss (SL)~\cite{xu2018semantic} and Semantic probabilistic Layer (SPL)~\cite{ahmed2022spl}) as well as Scallop which employs a top-$k$ approximaization are end-to-end differentiable and rely on knowledge compilation, which limits scalability. They are ill‑suited for the path‑planning task, where shortest‑path computation forms a large combinatorial argmin problem that does not align well with logic‑based differentiable inference.
A‑NeSI replaces symbolic inference with a neural approximation, but on MNIST‑Addition we observe overflow when scaling to $100$ digits.
I‑MLE and EXAL are more closely related to EM‑NeSy in that they rely on sampling‑based training objectives. However, neither I‑MLE nor EXAL is applicable to the Visual‑Sudoku task, as both require either a MAP solver or an efficient sampler over valid Sudoku solutions conditioned on predicted digit probabilities -- which is prohibitively expensive.
In EM‑NeSy, we use multiple inference approaches: sum–product belief propagation (BP‑EM), exact for MNIST‑Addition and approximate for Visual Sudoku due to loopy underlying graphs; max–product belief propagation (BP‑EM‑M), yielding hard EM; and ABC sampling (EM‑ABC). We additionally include an end‑to‑end differentiable baseline with the same sum–product belief propagation inference (BP‑std).

\subsection{Results}

In this section, we present our experimental results; for more details, see Appendix \cref{exp_details}.

\begin{table*}
        \centering
        \caption{Accuracy and training time on the MNISTAdd task.}
        \label{table_exact}
        \hspace{2.3cm}\textbf{accuracy in $\%$} \hspace{4.9cm} \textbf{training time in $min$}\\
        \begin{tabular}{l|ccc|ccc}
            \hline
            number of digits & $4$ & $15$ & $100$ & $4$ & $15$ & $100$\\
            \hline
            A-NeSI & $93.28 \pm 0.25$ & $55.88 \pm 9.30$ & -- & $25.75$ & $57.55$ & --\\
            EXAL & $90.71 \pm 1.01$ & $62.62 \pm 4.24$ & $6.67 \pm 7.83$ & $4.17$ & $42.44$ & $372.20$\\
            DSL\footnotemark \small{(exact)}
                & $92.70 \pm 0.60$ & T/O & T/O & -- & T/O & T/O\\
             BP-std \small{(exact)} & $92.16 \pm 0.23$ & $\mathbf{72.79 \pm 1.62}$ & $11.60 \pm 3.85$ & $2.21$ & $7.93$ & $36.26$\\
             \hline
             BP-EM \small{(ours)} & $92.53 \pm 0.62$ & $\mathbf{74.85 \pm 1.83}$ & $9.20 \pm 9.76$ & $\mathbf{1.33}$ & $\mathbf{1.95}$ & $\mathbf{8.51}$\\
             BP-EM-M \small{(ours)} & $33.44 \pm 33.49$ & $30.36 \pm 40.09$ & $0.00 \pm 0.00$ & $\mathbf{1.33}$ & $\mathbf{2.38}$ & $13.09$\\
        \end{tabular}
        \end{table*} 

    \begin{table*}
        \centering
        \caption{Accuracy and training time on the Visual Sudoku task.}
        \label{table_Sudoku}
        \textbf{\hspace{3.4cm} accuracy in $\%$ \hspace{2.3cm} training time in $min$}\\
        \begin{tabular}{l|cc|cc}
            \hline
            Sudoku size & $4\times4$ & $9\times9$& $4\times4$ & $9\times9$\\
            \hline
            A-NeSI & $87.20 \pm 2.18$ & $\mathbf{59.20 \pm 2.06}$ & $46.08$ & $73.33$\\
             Scallop & $75.00 \pm 0.45$ & T/O & $\mathbf{1.52}$ & T/O \\
             SL\footnotemark[\value{footnote}] \small{(exact)} & $86.70 \pm 0.50 $ &  T/O & -- & T/O\\
             BP-std & $87.80 \pm 2.05$ & $0.50 \pm 0.00$ & $38.23$ & --\\
             \hline
             BP-EM \small{(ours)} & $89.70 \pm 3.05$ & $0.50 \pm 0.00$ & $21.23$ & --\\
             ABC-EM \small{(ours)} & $86.30 \pm 1.30$ & $53.30 \pm 3.15$ & $\mathbf{1.40}$ & $\mathbf{8.24}$\\
        \end{tabular}
    \end{table*}

    \begin{table*}
        \centering
        \caption{Accuracy and training time on the Warcraft path-planning task.}
        \label{table_warcraft}
        \textbf{\hspace{2cm} accuracy in $\%$ \hspace{2.4cm} training time in $min$}\\
        \begin{tabular}{l|cc|cc}
            \hline
            grid size & $12\times 12$ & $30\times 30$& $12\times 12$ & $30\times 30$\\
            \hline
            A-NeSI & $\mathbf{98.96 \pm 1.33} $ & $67.57 \pm 36.76$ & $439.10$ & $1596.51$\\
            I-MLE & $95.34 \pm 0.20$ & $\mathbf{93.4 \pm 0.64}$ & $26.77$ & $227.82$\\
            EXAL\footnotemark[\value{footnote}] & $94.19 \pm 1.74$  & $80.85 \pm 3.83$ & $\mathbf{11.1 \pm 0.1}$ & $84.3 \pm 0.7$\\
            SPL\footnotemark[\value{footnote}] \small{(exact)} & $78.2 $ & T/O & -- & T/O \\
             \hline
             ABC-EM (ours) & $\mathbf{98.80 \pm 0.41}$ & $69.60 \pm 35.64$ & $\mathbf{7.26}$ & $\mathbf{8.37}$\\
        \end{tabular}
    \end{table*}
    \footnotetext{Accuracies for Semantic Loss and Semantic Probabilistic Layer are taken from~\cite{krieken2023anesi}, accuracies and timings for EXAL on Warcraft path-planning and accuracies for DeepStochLog from~\cite{verreet2024exal}. Note that EXAL was executed on a different machine.}

    \textbf{Q1}: To evaluate EM‑NeSy under exact inference, we use the MNISTAdd benchmark and compare EM‑NeSy (BP‑EM) with the standard NeSy learning setting (BP‑std), both using sum‑product belief propagation\footnote{Belief propagation operates on a tree in this task and is therefore exact.}, and with DeepStochLog.
    The corresponding results are reported in~\cref{table_exact}.
    As predicted by \cref{equivproof}, accuracies are nearly identical, with minor numerical differences resulting in increased variance for EM‑NeSy at 100 digits.
    Knowledge compilation prevents DeepStochLog from scaling to 15 digits, whereas EM‑NeSy scales to 100 digits and is faster than end‑to‑end NeSy since it avoids backpropagation through iterative updates and gradient‑checkpointing overhead.
    This also leads to a substantial reduction in peak memory usage, from $90.61$\,GB to $35.52$\,GB at 100 digits.
    Hence, the answer to Q1 is that under exact inference, EM-NeSy matches the accuracy of the standard end-to-end NeSy setting while offering faster training and decreased memory usage.

    \textbf{Q2}:
    To evaluate EM‑NeSy beyond exact inference, we instantiate the E‑step with several approximate or non‑differentiable inference procedures. Our experiments show that, unlike many other methods, EM‑NeSy can easily incorporate different inference procedures directly, without requiring differentiation through the symbolic component, enabling a direct comparison between them. Moreover, it scales to all experimental settings where exact inference times out. However, learning performance depends strongly on the quality of the approximate posterior, thereby inheriting the strengths and limitations of the underlying inference method. Hard EM (BP‑EM‑M) often converges to poor local optima on MNISTAdd, likely due to the bias introduced by MAP inference. Loopy belief propagation (BP‑EM) fails on the $9\times9$ Sudoku grid because of the highly loopy factor graph structure. ABC sampling (ABC-EM) exhibits high performance variance on the larger instances across all three benchmarks. Nonetheless, it is noteworthy that ABC sampling achieves state‑of‑the‑art results in all smaller settings, and in some runs on the larger ones, while keeping training time minimal, despite being a considerably more general inference technique.
With respect to Q2, these results show that EM‑NeSy can be applied in settings where end‑to‑end differentiation is unavailable or impractical, while remaining computationally efficient. Overall, EM‑NeSy is robust to moderately approximate E‑steps, but increasing inference noise leads to higher variance and may eventually lead training to break down.
    

\section{Conclusion}
    \label{sec_conlusion}
EM-NeSy shows that NeSy learning can be viewed as an instance of the EM algorithm.
In contrast to state‑of‑the‑art NeSy approaches -- which typically require inference‑specific modifications to ensure differentiability under approximate inference -- EM‑NeSy provides a unified learning framework that supports arbitrary approximate probabilistic inference methods by eliminating the need for differentiable symbolic inference. This substantially simplifies model design and enables the development of more scalable and flexible NeSy systems.

\paragraph{Limitations and future work}
    While EM‑NeSy accommodates both exact and approximate inference without differentiability requirements, it does not address how to select effective approximations, and our experimental results depend on task‑specific hyperparameter tuning.

\section*{Acknowledgements}
This research received funding from the Flemish Government (AI Research Program), from the Flanders Research Foundation (FWO) under project G047124N and from the European Research Council (ERC) under the European Union's Horizon Europe research and innovation programm (grant agreement No. 101142702).

\begin{small}
\bibliography{ref}
\bibliographystyle{plain}
\end{small}

\appendix

\appendix
\section{Proof of \cref{equivproof}}
\label{app_equivproof}
    \begin{proof}
        Recall that, by the chain rule, we have
        \begin{equation}
            \label{log_chain_rule}
            \nabla_{\boldsymbol \theta} \log f(\boldsymbol \theta) = \frac{\nabla_{\boldsymbol \theta} f(\boldsymbol \theta)}{f(\boldsymbol \theta)}
        \end{equation}
        for any differentiable function $f$.
        Hence, the claim follows from the following sequence of equalities:
        \begin{align*}
            &\nabla_{\boldsymbol \theta} \mathcal L(\boldsymbol \theta) \overset{\eqref{std_loss}}= - \nabla_{\boldsymbol \theta} \log p(\mathbf y|\mathbf x,\boldsymbol \theta) \\
            &\overset{\eqref{log_chain_rule}}= -\frac{\nabla_{\boldsymbol \theta} p(\mathbf y|\mathbf x,\boldsymbol \theta)}{p(\mathbf y|\mathbf x,\boldsymbol \theta)}\\
            &\overset{\eqref{std_loss}}= -\frac {\nabla_{\boldsymbol \theta} \sum_{\mathbf z \in R(\mathbf Z)} p_{\text{Ne}}(\mathbf z|\mathbf x,\boldsymbol \theta)p_{\text{Sy}}(\mathbf y|\mathbf z)}{p(\mathbf y|\mathbf x,\boldsymbol \theta)}  \\
            &= -\frac {\sum_{\mathbf z \in R(\mathbf Z)} p_{\text{Sy}}(\mathbf y|\mathbf z) \nabla_{\boldsymbol \theta} p_{\text{Ne}}(\mathbf z|\mathbf x,\boldsymbol \theta)} {p(\mathbf y|\mathbf x,\boldsymbol \theta)}  \\
            &\overset{\eqref{log_chain_rule}}= -\frac {\sum_{\mathbf z \in R(\mathbf Z)} p_{\text{Sy}}(\mathbf y|\mathbf z) p_{\text{Ne}}(\mathbf z|\mathbf x,\boldsymbol \theta) \nabla_{\boldsymbol \theta} \log p_{\text{Ne}}(\mathbf z|\mathbf x,\boldsymbol \theta)} {p(\mathbf y|\mathbf x,\boldsymbol \theta)}  \\
            &\overset{\eqref{eq_NeSy_pred}}= -\frac {\sum_{\mathbf z \in R(\mathbf Z)} p(\mathbf z,\mathbf y|\mathbf x,\boldsymbol \theta) \nabla_{\boldsymbol \theta} \log p_{\text{Ne}}(\mathbf z|\mathbf x,\boldsymbol \theta)} {p(\mathbf y|\mathbf x,\boldsymbol \theta)}  \\
            &\overset{\eqref{eq_latent_posterior}}= -\frac {\sum_{\mathbf z \in R(\mathbf Z)} p(\mathbf y|\mathbf x,\boldsymbol \theta) p(\mathbf z|\mathbf y,\mathbf x,\boldsymbol \theta) \nabla_{\boldsymbol \theta} \log p_{\text{Ne}}(\mathbf z|\mathbf x,\boldsymbol \theta)} {p(\mathbf y|\mathbf x,\boldsymbol \theta)}  \\
            &= -\sum_{\mathbf z \in R(\mathbf Z)} p(\mathbf z|\mathbf y,\mathbf x,\boldsymbol \theta) \nabla_{\boldsymbol \theta} \log p_{\text{Ne}}(\mathbf z|\mathbf x,\boldsymbol \theta)\\
            &\overset{\eqref{pseudo_label}}= -\sum_{\mathbf z \in R(\mathbf Z)} \ell(\mathbf z) \nabla_{\boldsymbol \theta} \log p_{\text{Ne}}(\mathbf z|\mathbf x,\theta)\\
            &= \nabla_{\boldsymbol \theta} \mathcal L_{\text{EM-NeSy}}(\theta).
        \end{align*}
    \end{proof}
\section{Multi-M EM}
    \label{app_multim}

    \begin{algorithm}
\caption{One training step of Multi-M EM}
\begin{algorithmic}
\REQUIRE training sample $(\mathbf x, \mathbf y)$, number of $M$ steps $m$
\STATE Compute $p_{\text{Ne}}(\mathbf z| \mathbf x,\boldsymbol \theta)$
\STATE Compute $p_{\text{Sy}}(\mathbf z | \mathbf y, \mathbf x, \boldsymbol \theta)$
\STATE For $i=1$ to $m$: Backpropagate w.r.t. $\mathcal L_{\text{EM-NeSy}}(\boldsymbol \theta)$
\end{algorithmic}
\end{algorithm}
    
\section{Experimental details}
    \label{exp_details}

    All experiments were performed using a single NVIDIA L40S GPU with 48\,GB of memory on a machine equipped with a 128-thread CPU and 256\,GB of RAM.
    For all experiments, we report the mean and standard deviation of the accuracies over random seeds $0,\dots,4$.
    We compute training times by 

    \subsection{Benchmarks}
    \label{benchmarks}

    \paragraph{MNISTAdd} The MNIST dataset~\cite{deng2012mnist} contains $60{,}000$ training and $10{,}000$ test examples. We further split the original training set into a training and a validation split of sizes $50{,}000$ and $10{,}000$, respectively. From each of the resulting three MNIST splits and for a given number of digits $n$, we construct MNISTAdd datasets of size $\frac{50{,}000}{2n}$. Each MNIST digit is used in at most one example, where every example consists of $2n$ MNIST images representing two $n$-digit numbers, and the corresponding label is their sum.

    \paragraph{Visual Sudoku}
    \label{app_visudo}
    The predefined Visual Sudoku datasets \cite{augustine2022visual} consist of $500$ Sudoku grids of size either $4\times4$ or $9\times9$, split into $100$ training examples, $200$ validation examples, and $200$ test examples.
    A total of $11$ splits are provided. We use the first one for all our experiments.
    Within each split, there is no overlap between the MNIST images used in the training, validation, and test sets.
    All training puzzles are valid Sudokus, whereas in the validation and test sets half of the puzzles are invalid.
    Invalid puzzles are generated by corrupting valid ones, either by replacing a correct digit with an incorrect one or by swapping two valid digits.

    \paragraph{Warcraft path-planning}
    The predefined Warcraft path-planning datasets~\cite{vlastelica2019differentiation} consist of grid-based maps of size $12\times12$ or $30\times30$. The maps are generated using a custom random sampling procedure based on $142$ terrain tiles from the \emph{Warcraft~II} game environment. Each tile belongs to one of five terrain types, each associated with a fixed traversal cost ranging from $0.8$ to $9.2$. For each map, the supervision is given as a $12\times12$ or $30\times30$ binary matrix indicating the shortest path from the top-left to the bottom-right corner. The training, validation, and test splits contain $10{,}000$, $1{,}000$, and $1{,}000$ examples, respectively.

    \subsection{Neural component}
    For MNIST-digit recognition, we use a LeNet-style convolutional neural network~\cite{lecun1998gradient} consisting of two convolution--max-pooling blocks with $5\times5$ kernels and ReLU activations, followed by three fully connected layers.
    Given inputs of size $28\times28$, the network produces $N$-dimensional class probabilities for each digit.
    For inputs containing multiple digits, the same network is applied independently to each digit via weight sharing.

    To extract the latent costs of the terrain tiles in the Warcraft path-finding task, we use the first five layers of ResNet18~\cite{he2016deep} followed by a max-pooling operation. The resulting continuous value for each grid cell is then mapped to a probability distribution over the five terrain types by computing the $L^2$ distances to learnable class centers (initialized with $0.0, 1.0, 2.0, 3.0, 4.0$) and applying a softmax transformation with temperature $3$.

    \subsection{Symbolic component}

    \paragraph{Belief propagation}
    We use our own implementation of Pearl's loopy belief propagation algorithm \cite{pearl1994belief}.
    It employs a synchronized message‑passing scheme in which all messages are updated in parallel.

    For MNISTAdd, the underlying factor graph contains $2n$ categorical variables for each input digit and $n$ categorical variables for each sum digit, as well as $n+1$ binary carry variables. 
    The resulting graph is a tree, so belief propagation yields exact marginals. 
    We perform $2n$ iterations of belief propagation.

    For Visual Sudoku, the factor graph contains $n^2$ categorical variables for ($n \in\{4,9\}$) for every cell of the Sudoku grid and $m = 2 \cdot n \cdot \binom n 2 + \frac{n^2 \cdot (\sqrt n - 1)} 2 \in \{56,810\}$ binary variables for each of the $m$ pairs of cells that share a row, column, or box. 
    In addition, $m-1$ binary variables encode constraints that combine subsets of $2$ to $m$ of the $m$ pair constraints. We perform $m$ iterations of belief propagation, using a damping parameter of $10^{-2}$ for the message updates. To avoid numerical degeneracies in deterministic factors where message entries may collapse to zero we apply a clamping parameter of $10^{-5}$.

    \paragraph{ABC rejection sampling}
    \label{app_abc}

    Both in the Visual Sudoku and Warcraft path-planning task, we follow the approach described in \cref{abc}. We sample uniformly from the neurally predicted distribution $p_{\text{Ne}}(\cdot|x,\theta)$ for the Sudoku digits or Warcraft grid cells, which factorizes into the distributions for the single digits/grid cells $Z_i$ for $i=1,\dots,n$.
    
    For the Visual Sudoku experiment, let $m_{ij}$ denote the constraint $Z_i \ne Z_j$.
    As described in \cref{example_sudoku}, we define $$M:=\{m_{ij}: Z_i,Z_j \text{ share a row column or box in the Sudoku}\}$$ to be the set of Sudoku constraints. For every fixed $i$ let $$M_i := \{m_{ij} \in M : j \in \{1,\dots,n^2\}\}$$ be the set of all constraints that involve $Z_i$.
    We further choose the distance measure $\rho = L^1$ and the summary statistics $s_{\mathbf Y}(\mathbf y) = 0$ and $s_{Z_i}(z_i) = \frac{k}{|M_i|}$ where $C$ is a normalization constant and $k$ denotes the number of constraints from $M_i$ that are satisfied by the sample $z_i$.
    Then we assign the weight
    $$
    \frac 1 C \rho\bigl(s_{Z_i}(z_i), s_{\mathbf Y}(\mathbf y)\bigr)
    = \frac 1 C \left| \frac{k}{|M_i|} - 0 \right|
    = \frac{k}{C|M_i|},
    $$ to each sample $z_i$ where $C$ denotes a normalization constant.

    In the Warcraft experiment, we employ a Dijkstra solver that predicts the shortest path for each sample. We let $s_{Z_i}(z_i) = 1$ if the predicted shortest path for the sample $\mathbf z$ matches the true shortest path on the grid cell $Z_i$ (i.e.\ both the predicted and true shortest path contain $Z_i$ or they both don't).
    We let $s_{\mathbf Y(\mathbf y)} = 0$ and assign the weight $
    \frac 1 C \rho\bigl(s_{Z_i}(z_i), s_{\mathbf Y}(\mathbf y)\bigr)
    $ to the sample $z_i$.

    \subsection{Hyperparameters}
    \label{hyperparam}

    We used the validation splits for hyperparameter tuning.
    At test time, we report results using the checkpoint that achieves the highest validation accuracy.
    We employ the Adam optimizer for all experiments.
    
    For MNISTAdd, we use a cosine‑annealing learning‑rate schedule updated after each epoch, beginning at $10^{-3}$ and decaying to $10^{-4}$.
    We train for $30$ epochs with batch sizes of $50$, $10$ and $2$ for $4$, $15$ and $100$ digits, respectively.

    For Visual Sudoku, we use a cosine‑annealing learning‑rate schedule updated after each epoch, beginning at $10^{-3}$ and decaying to $10^{-6}$.
    We train with a batchsize of $5$ for $500$ and $200$ epochs for $4\times4$ and $9\times9$ Sudokus, respectively.
    For ABC sampling, we use $1000$ samples on $4\times4$ Sudokus and $10{,}000$ for $9\times9$ Sudokus.

    For the Warcraft path-planning experiment, we use a constant learning‑rate of $3\cdot 10^{-4}$ on the $12\times12$ grid and $10^{-4}$ on the $30\times30$ grid.
    We train with a batchsize of $50$ for $5$ and $10$ epochs, respectively.
    We use an entropy regularization parameter of $0.03$ for training on the $12\times12$ grid.
    For ABC sampling, we used $50$ and $20$ samples on $12\times12$ and $30\times30$ grids, respectively.
    
    \subsection{Baselines}

We ran the experiments for A‑NeSI~\cite{krieken2023anesi}, I‑MLE~\cite{niepert2021implicit},
EXAL~\cite{verreet2024exal} and
Scallop~\cite{huang2021scallop} with the hyperparameters specified in the original papers.
For A-NeSI we use the symbolic prediction variant.
For A-NeSI on the Warcraft path-planning benchmark, we were not able to reproduce the results due to high variance and training times, so we cite the accuracies from the original paper and report training time on our machines for a single run.
For EXAL on the Warcraft path-planning benchmark, we report training times and accuracies from the original paper. Note that these experiments were performed on a another machine, namely a Dell XPS 15 (i7, 16GB, 512GB, FHD, GPU).



\end{document}